\documentclass{article}

\usepackage[preprint,nonatbib]{neurips_2020}

\usepackage[utf8]{inputenc} 
\usepackage[T1]{fontenc}    
\usepackage{hyperref}       
\usepackage{url}            
\usepackage{booktabs}       
\usepackage{amsfonts}       
\usepackage{nicefrac}       
\usepackage{microtype}      
\usepackage{color}
\usepackage{mathtools}
\usepackage{graphicx}
\usepackage{amsmath}
\usepackage{mathtools}
\usepackage{amssymb}
\usepackage[toc,page]{appendix}

\usepackage{bmpsize}
\usepackage{subfig}

\newtheorem{lemma}{Lemma}
\newtheorem{proof}{Proof}
\newtheorem{theorem}{Theorem}

\title{StarNet: Gradient-free Training of Deep Generative Models using Determined System of Linear Equations}

\author{%
  Amir Zadeh, Santiago Benoit, Louis-Philippe Morency \\
  LTI, Computer Science Department\\
  Carnegie Mellon University\\
  Pittsburgh, PA 15213 \\
  \texttt{\{abagherz,sbenoit,morency\}@cs.cmu.edu} \\
}

\begin{document}
\maketitle

\begin{abstract}
In this paper we present an approach for training deep generative models solely based on solving determined systems of linear equations. A network that uses this approach, called a StarNet, has the following desirable properties: 1) training requires no gradient as solution to the system of linear equations is not stochastic, 2) is highly scalable when solving the system of linear equations w.r.t the latent codes, and similarly for the parameters of the model, and 3) it gives desirable least-square bounds for the estimation of latent codes and network parameters within each layer. 
\end{abstract}

\section{Introduction}
Generative modeling requires learning a latent space, and a decoder that maps the latent samples to an output manifold. We study the two cases of using inverse-funnel feedforward decoders and convolutional decoders through the lens of a system of linear equations. Our proposed algorithm is called StarNet. 

A system of liner equations is a well-studied area in linear algebra, with various algorithms ranging from elimination methods to solutions based on matrix pseudoinverse. The benefits of StarNet learning process are as follows:

\begin{itemize}
    \item No gradients are required to solve a system of liner equations. Therefore, training can be done at a lower computational cost, and fewer hyperparameters. Training deep neural networks often requires calculating the gradient of an objective function (e.g. log likelihood) w.r.t the parameters of a model. Gradient calculation is computationally heavy, and often scales poorly since batch updates are serial throughout an epoch. 
    
    \item Certain algorithms to solve a system of linear equations can be highly scalable. Unlike batch learning, where each datapoint waits its turn within the epoch, a system of linear equations offers a framework that is more scalable. Certain steps of the StarNet learning algorithm can linearly scale up to as many computational nodes as there are datapoints within a dataset. 
    
    \item There are strong theoretical backings for a solution to a system of linear equations. It is imperative that deep models have explainable behavior, during both learning and application. Learning via a system of linear equations offers residuals for each layer, which can be used to identify and better approach multiple important issues in machine learning such as mode-collapse. Furthermore, for each datapoint, an estimate of how well each layer fits the datapoint can be measured, thus better understanding the intrinsic success or failure cases of deep learning. 

\end{itemize}

A system of linear equations offers unique (exact/approximate) solutions if it is determined (well-determined/over-determined). This paper presents such conditions for two commonly used decoder types: feedforward and convolutional. These conditions are not strong, and a significant portion of the currently used decoders fit within these conditions. Our StarNet learning algorithm is defined by progressively training a deep neural network from the last layer to the first. It follows a coordinate descent method to reach a plateau in solving the system of linear equations w.r.t both the latent codes and the parameters of the model, at each layer. Most of the trained models in this paper reach convergence in less than $3$ epochs - a significant boost over gradient models. Due to the fact that the latents of the StarNet are not restricted (e.g. an encoder in Auto-Encoder or PCA), even a single layer StarNet offers remarkable generation capabilities.
\section{Related Work}\label{sec:related}

Generative modeling is a fundamental research area in machine learning. Notable related works re outlined below. 

In a neural framework, Auto-Encoders (AE) and their variational implementation (VAE~\cite{kingma2013auto}) have been commonly used for learning based on a reconstruction objective. GANs~\cite{goodfellow2014generative} deviate from this framework by optimizing an objective that pushes two networks towards an equilibrium where a generators output is indistinguishable from the real generative distribution. An encoderless application of VI to generative modeling, Variational Auto-Decoders (VAD~\cite{zadeh2019variational}) use only a generator but with an objective similar to VAE - with the added benefit of robustness to noise and missing data.  All approaches rely on computation of gradients during train time (with GAN and VAD commonly requiring the same during test time to invert the generative model, i.e. posterior inference). Flow-based models use invertible mappings between input and output of same dimensions. Convolutional variants are limited to $1 \times 1$ convolutions~\cite{kingma2018glow}. StarNet does not require same dimensions for input and output, and is not restricted to $1 \times 1$ convolutions. However, similar to the latter, StarNet does exhibit certain mild limitations when solving for CNNs. 

Aside deep techniques, Principle Component Analysis (PCA~\cite{wold1987principal}) is a popular method for learning low-dimensional representation of data. PCA is identical to a linear auto-encoder with a lasso reconstruction objective. PCA algorithm does not requires gradient computations. It differs from the StarNet in that the feed forward operation of PCA is restricted by a linear encoder, while the StarNet imposes no restrictions on the latent space. In simple terms, StarNet has no direct mapping between input and latents; the latents are recovered using pesudoinverse of the model weights.

\section{Model}\label{sec:model}
StarNet is a method for training certain neural decoders without gradient computations. We cover two major types in this paper: 1) feedforward StarNet, and 2) convolutional StarNet. For each type, we outline the necessary conditions required for successful training and inference. 

\subsection{Feedforward StarNet}

We first start by a feedforward StarNet, which uses a feedforward decoder. The following are the operations in tthe $i$-th layer of a neural network:

\begin{equation}\label{eq:nn}
    h_i=a\ (W_i \ \cdot \ h_{i-1})
\end{equation}

With $h_i \in \mathbb{R}^{d_i}$ as output, $h_{i-1} \in \mathbb{R}^{d_{i-1}}$ as input, $W \in \mathbb{R}^{d_i \times d_{i-1}}$ as an affine map, $a$ as a non-linear activation function. Note, for simplicity of equations later, the formula accounts for bias as a constant dimensions added to all $h_{i-1}$ and a row added to $W$. Such operation above is unfolded $K$ times for a deep neural decoder wih $K$ layer. 

Let $X=\{x^{(n)}\}_{n=1}^N$ be a given i.i.d dataset. Essentially, in any generative or unsupervised framework, the hope is that $h^{(n)}_{i=K}$ is as close as possible to the respective $x^{(n)}$. 
\begin{theorem}
    $\forall i$, Equation \ref{eq:nn} is solvable w.r.t $h_{i-1}$ as long as $d_{i-1}\leq d_i$ (i.e. decoder is inverse funnel). This is coupled with $a^{-1}$ as the inverse of the activation.
\end{theorem}
\begin{proof}\label{proof:solvez}
    The condition is required for well-determined, or over-determined solution to a system of liner equations w.r.t $h_{i-1}$. Let $\hat{h}_{i}$ be the non-activated neural responses in Equation \ref{eq:nn}. Naturally, $\hat{h}_{i} = a^{-1} (h_{i})$. Foreach $x^{(n)}$, we have the following:
    \begin{equation}\label{eq:solvez}
        W_i \ \cdot \ h^{(n)}_{i-1} = \hat{h}_{i-1}^{(n)}
    \end{equation}
    The above is a system of linear equations with a unique solution in the case of $d_{i-1} = d_i$ and a least-squared solution in the case of $d_{i-1} < d_i$. The solution is exactly as:
    
    \begin{equation}
        \hat{h}^*_{i-1}=(W_i^T W_i)^{-1} W_i^T \hat{h}_{i}
    \end{equation}
    
    $\hat{h}^*_{i-1}$ is the least-squared solution (hence the name StarNet). $(W_i^T W_i)^{-1} W_i^T$ is the pseudo-inverse of the weights of the $i$th layer. It is important that the initialization of $W$ does not produce rank-deficiency. Therefore, random initialization or its orthogonal projection (product of a decomposition) may work. Subsequently:
    
    \begin{equation}\label{eq:pseudo}
        h^*_{i-1}=a^{-1}(\hat{h}^*_{i-1})
    \end{equation}
    
    Commonly used activation functions are invertible, except for ReLU in the negative part of its domain. Alternatives such as leaky ReLU can be used instead. 
    
\end{proof}
\begin{lemma}
    The above will not have a unique solution if $d_{i-1} > d_i$, since $nullity(W)>0$. Therefore, the architecture cannot go lowerin dimensions without the risk of rank deficiency. If so, strong assumptions on the observations $x^{(n)}$ have to be made, such as presence of a particular structure or certain simplicities\cite{donoho2006simplest}, often covered in compressed sensing \cite{donoho2006compressed}. In the scope of this paper, since majority of the neural decoders are inverse funnels, we leave the studies of rank deficient weights for the next edition of this preprint. 
\end{lemma}
\begin{lemma}
    Since samples drawn from $p(x)$ are assumed i.i.d, the solution to the Equation \ref{eq:solvez} can be distributed across $N$ independent computations. Therefore, allowing for linear job distribution over computational nodes. 
\end{lemma}

\begin{lemma}
    The solution of the Equation \ref{eq:nn} for a datapoint can be seen as the latent inference process (e.g. similar to an encoder feedforward process), which is part of both training and testing. It can be applied recursively from higher to lower layers, with a succession of solving Equation \ref{eq:nn} for layers $K-1,K-2,\cdots,1$, and applying inverse of the activation $a$ at each layer. Therefore, the complexity of inference in StarNet remains linear w.r.t the number of layers $K$. This is an improvement over encoderless gradient-based approaches that require large number of epochs (e.g. approximate posterior inference in VAD). 
\end{lemma}

\begin{lemma}
    If Equation \ref{eq:solvez} is over-determined, then a norm of the residuals of each equation (i.e. datapoint when solving for $h^{(\cdot)}$) can be used as a measure for confidence, detecting outliers, or mode-collapse. An in-distribution datapoint with high residual norm would signal inference suboptimality given the layer parameters. Similarly, a datapoint may be judged to belong to a certain distribution using a probability field defined over such a measure. 
\end{lemma}

\begin{lemma}
    A layer which has $d_i=d_{i-1}$ is not a useful neural construct for StarNet. Such a layer will collapse an affine transformation to a degenerate identity projection. If a solution lies within the $R^{d_{i-1}}$, then the Equation \ref{eq:pseudo} will find it, regardless of the depth of similarly shaped layers projecting into that same space. 
\end{lemma}

\begin{theorem}
Equation \ref{eq:nn} can also be solved w.r.t $W_i$ iff $d_{i-1} \leq N$; essentially, there needs to be more datapoints than the number of input neurons in each individual layer of the network. 

\end{theorem}

\begin{proof}
Let $W_{i,j}$ be the $j$th row (that produces the $j$th dimension in the ouput of the layer) of the weights of the $i$th layer. Let $H_{i-1}$ be the matrix of inputs (over the entire dataset) to the $i$th layer. Thus equation \ref{eq:nn} can be written as:

\begin{equation}
    a^{-1}H_{i,j}^T= H_{i-1}^T \ \cdot \ W_{i,j}^T
\end{equation}

Hence a new system of linear equations with exactly $N$ equations and $d_{i-1}$ unknowns. The same equation can be written for each row of $W_i$ individually, leading to a scalable solution for $W_i$, across as many nodes as the total rows for $W_i$. A solution to this system can similarly be done by calculating the pseudoinverse of the knowns matrix, or elimination approaches.

\end{proof}

\begin{lemma}
    Since it is required to have the network follow an inverse funnel (See Proof \ref{proof:solvez}), then $d_{K-1} < N$ is the only condition that needs to be checked aside the previous conditions. 
\end{lemma}

\subsection{Convolutional StarNet}
Convolution is essentially a local affine map applied along patches of an input. Similar to feedforward StarNet, a system of linear equations can be solved for convolutional decoders. In order to keep the system of linear equations determined for the latents, we also discuss combining the convolution and pooling operations together for a conv-unpool layer. A conv-unpool layer swaps generates neighboring pixels of a large response map as channels within a smaller response map. Figure \ref{fig:convun} shows the operations of a conv-unpool layer. 

\begin{figure*}[t]
\centering{
\includegraphics[width=0.5\linewidth]{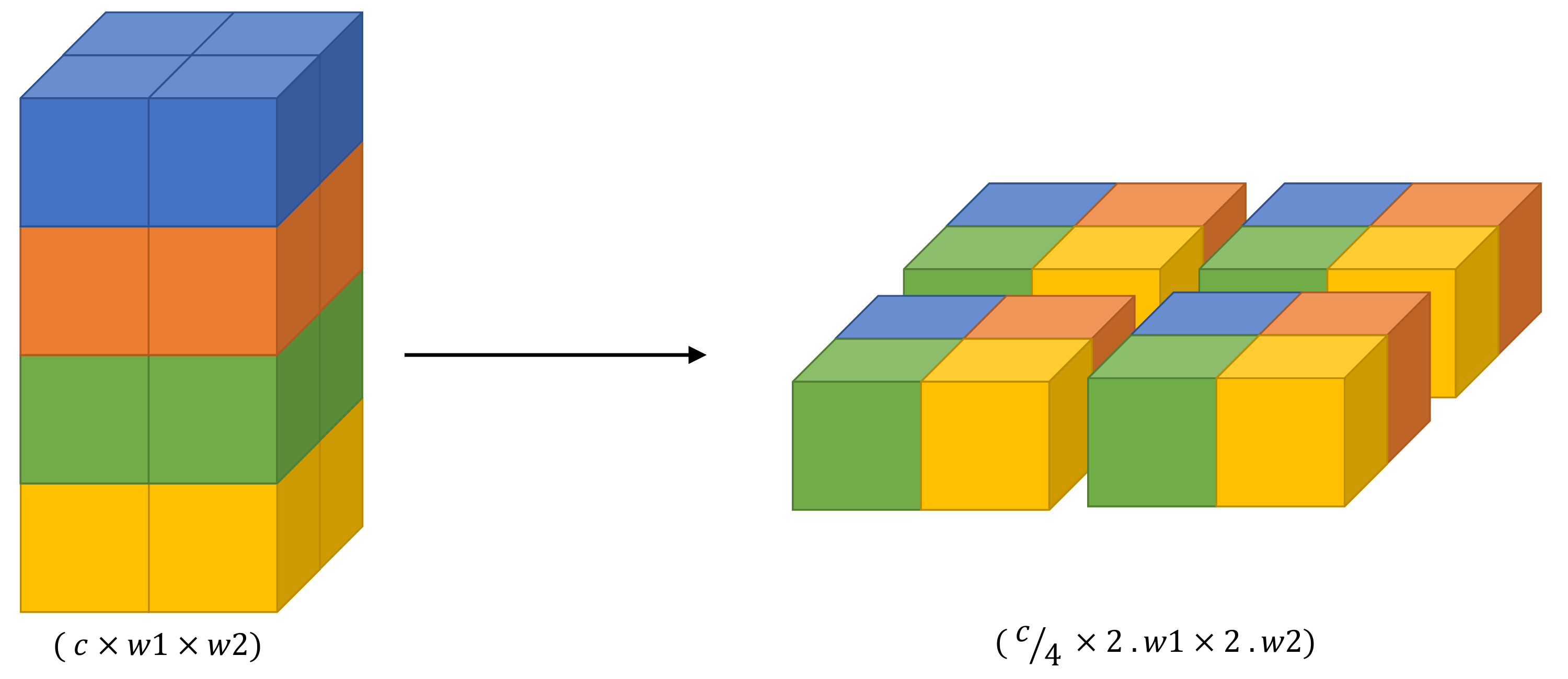}
\caption{\label{fig:convun}The operations of a conv-unpool layer with unpooling of $2 \times 2$.  Essentially a local patch (right) is created by reshaping the responses of the previous responses (left).}}
\end{figure*}

\begin{theorem}
    The system of linear equations for a conv-unpool layer with a stride of $1$ remains determined w.r.t the latents of the $(i-1)$th layer, as long as $c_i \leq c_{i-1} \cdot u_{i-1}^2$; with $c_i$, $c_{i-1}$ being the number of channels in the $i$th and $(i-1)$th layer of conv-unpool operations, and $u_{i}$ the unpooling shape in both image directions (here assumed to be square shaped, but can also be rectangle). 
\end{theorem}
\begin{lemma}\label{lemma:border}
    The above condition can be restrictive for the border convolutions. Performing full convolutions will help with the solution for all the inputs to remain determined (even the very borders of the input response map). 
\end{lemma}
\begin{theorem}
    The same system of linear equations can be solved w.r.t the convolutional parameters as long as there are fewer convolutional parameters in a single convolutional kernel (e.g. $49$ for a $7 \times 7$ convolutions) than the total number of convolutions for the same kernel on the entire dataset. This is somewhat a relaxed condition since layers tend to be larger than the shape of a single convolutional kernel. 
\end{theorem}
\begin{lemma}\label{lemma:subw}
    When solving for the convolutional parameters, the number of equations is extremely large compared to the number of unknowns (e.g. $2,880,000$ equations for a single $7 \times 7$ convolution in MNIST dataset with $49$ unknowns). This is due to the fact that the same kernels are applied over different patches. Naturally, calculation of the pseudoinverse may become computational exhaustive very quickly as the size of input image increases. We provide two practical solutions for this: 1) We can solve the system via a subset of the datapoints or equations. Naturally, this puts an emphasis on what datapoints are chosen. Hence, the subset should be reflective of the larger system of equations. In practice random datapoint or patch selection across the entire dataset works well, as shown in Figure \ref{fig:subw}, or 2) To get a more precise approximation of the convolutional weights, we can distribute the overdetermined equations across multiple computational nodes. Each computational node will still solve an overdetermined system, but a smaller one than the main system defined over the entire dataset. The final weights are subsequently the average of the solution found by the nodes.  
\end{lemma}

\subsection{Training and Inference Algorithm}

Training StarNet follows a coordinate descent framework. We start from the last layer of the network, where the $h^{(n)}_K=x^{(n)}$. We iteratively solve the system of linear equations w.r.t the latents, and subsequently w.r.t the weights of the model, until convergence is reached. The inverse activation is applied after training the layer to reach $h^{*(n)}_{K-1}$. The same approach is repeated for each layer, until we reach the first layer latents: $h^{*(n)}_{1}$. Test-time inference is identical to the layerwise training, except the system of equations in each layer is only solved for the latents, and never solved for the weights (unless the purpose is finetuning). 

Code is provided under \url{https://github.com/A2Zadeh/StarNet}.

\section{Experiments}\label{sec:exp}
In this section we discuss the experiments and the methodology. We first start by outlining the datasets, and subsequently discuss the training procedure and model architectures. 

\subsection{Datasets}

We experiment with the following four datasets.

\textbf{MNIST, Fashion-MNIST:} MNIST~\cite{lecun1998gradient} is a collection of handwritten digits, between $[0,9]$. Fashion-MNIST~\footnote{https://github.com/zalandoresearch/fashion-mnist} is a similar dataset to MNIST, except for fashion items. Both datasets are baselines for image generation. The images are monochromatic, with the shape of $28 \times 28$. 

\textbf{SVHN:} SVHN (Street View House Numbers) is a collection of real-world images of house numbers, obtained from Google Street View images~\cite{netzer2011reading}. One variant of this dataset includes bounding boxes around each digit in a house number, while the other variant (used in this paper) separates all digits individually into $32 \times 32$ images, with digits $[0,9]$ similarly to MNIST.  

\textbf{CIFAR10:} CIFAR10 is a collection of $32\times32$ color images belonging to 10 different classes, corresponding to different animals and vehicles~\cite{krizhevsky2009learning}. The full dataset consists of $6,000$ images per class, $60,000$ images in total.



\begin{figure*}[t]
\centering{
\includegraphics[width=\linewidth]{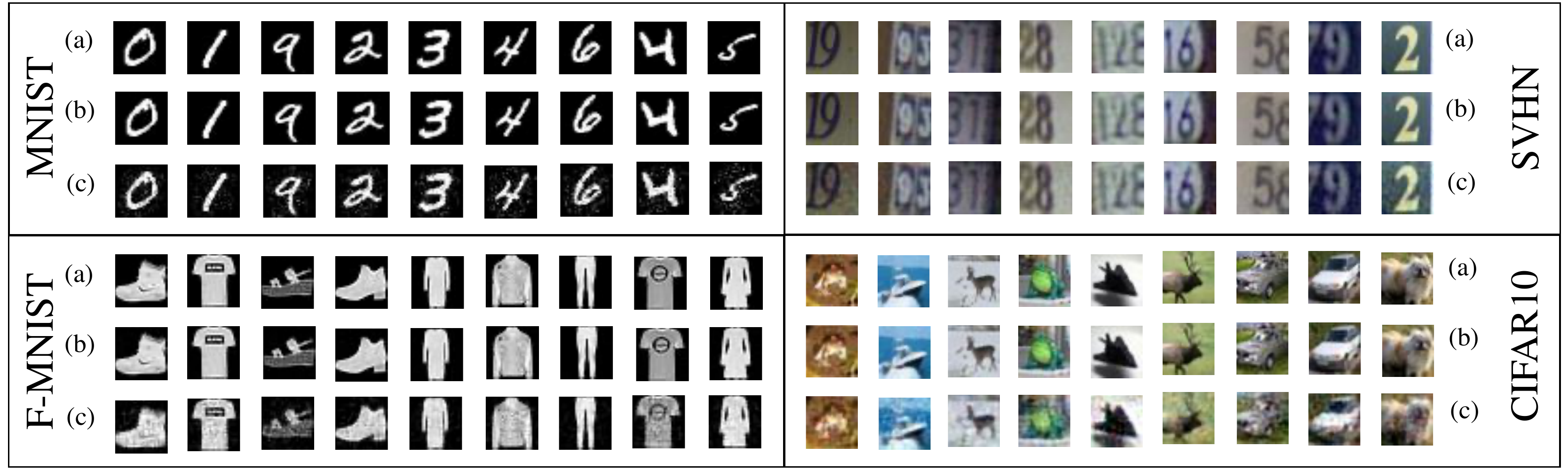}
\caption{\label{fig:allfc}Best viewed zoomed in and in color. Generation results for MNIST, Fashion-MNIST, SVHN and CIFAR10 using feedfoward StarNet with two layers. (a) denotes the target, (b) is the reconstruction after solving for the second layer of 256 neurons, (c) is the reconstruction after solving for the first layer of 128 neurons.}}
\end{figure*}
\begin{figure*}[t]
\centering{
\includegraphics[width=\linewidth]{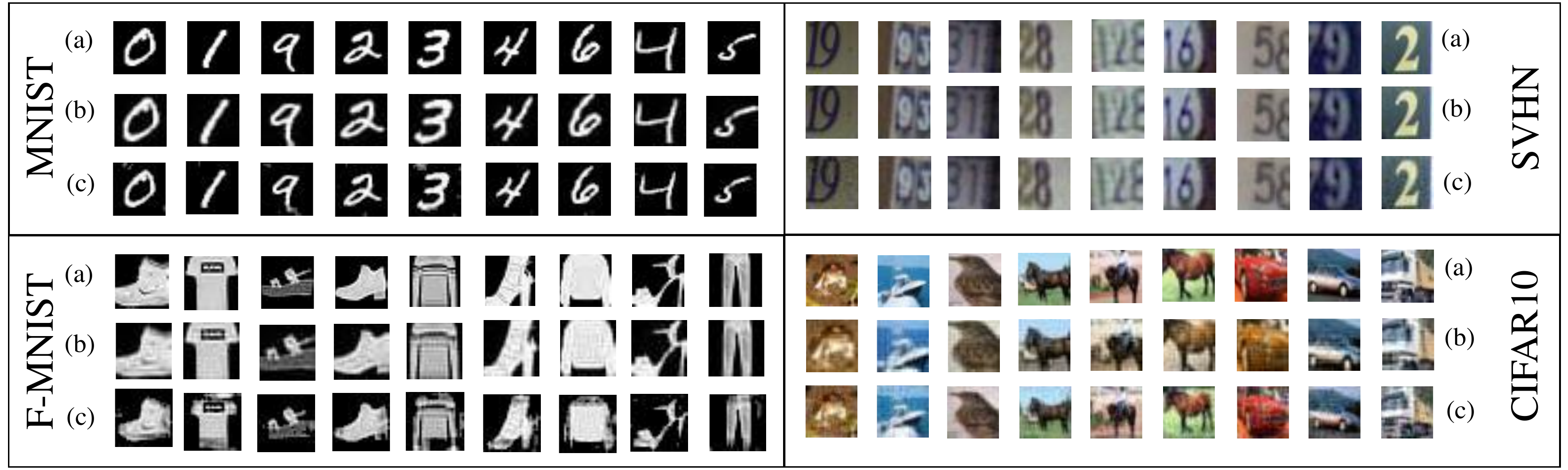}
\caption{\label{fig:allconv}Best viewed zoomed in and in color. Generation results for MNIST, Fashion-MNIST, SVHN and CIFAR10 using convolutional StarNet with two layers. (a) denotes the target, (b) is the reconstruction after solving for the second layer conv-unpool, (c) is the reconstruction after solving for the first layer conv-unpool.}}
\end{figure*}

\subsection{Training Methodology}\label{sec:training}
The training methodology for the feedforwad StarNet and convolutional StarNet are as follows:

\textbf{Feedforward StarNet:} We train feedforward networks with two layers of $128,256$ neurons. The final dimension for the MNIST and Fashion-MNIST is $784$ ($28 \times 28$), while SVHN and CIFAR10 are $3072$ ($3 \times 32 \times 32$). The activation is leaky ReLU, with the negative slope as a hyperparameter discussed in Section \ref{sec:results}. 

\textbf{Convolutional StarNet:} We use a conv-unpool operation, with the pooling of $2$ in each direction. For all trained models, there are two layers of convolution, with the shape $7 \times 7$. For MNIST and Fashion-MNIST, the first layer has $6$ kernels, and the second layer has $2$. For CIFAR10 and SVHN, we have $10$ kernels in the first layer, and $4$ kernels in the second. All trained models use full convolution to ensure that the system remains determined w.r.t the border latents (See Lemma \ref{lemma:border}). 

The dimensions chosen for the above feedforward and convolutional StarNet models follow the conditions required in Section \ref{sec:model}. Therefore, no degenerate identity solution can be learned. 

The comparison Auto-Encoder (AE) baseline uses the same architecture decoder. The encoder is the same architecture as the decoder (no weights shared between the encoder and decoder). Default Adam parameters are used for training with a learning rate of $10e-3$. Batch size is set to $64$.

\section{Results}\label{sec:results}

The generations for the datasets of MNIST, Fashion-MNIST, SVHN and CIFAR10 are shown in Figure \ref{fig:allfc} for feedforward StarNet, and Figure \ref{fig:allconv} for convolutional StarNet. While trained models use no gradients, the generations show high fidelity to ground truth\footnote{The high negative slope of 0.5 was found to work the best overall.}. We show the generations based on each layer, and further discuss the results of the single layer generation in Section \ref{sec:discuss}.

\begin{figure*}[t]
\centering{
\includegraphics[width=\linewidth]{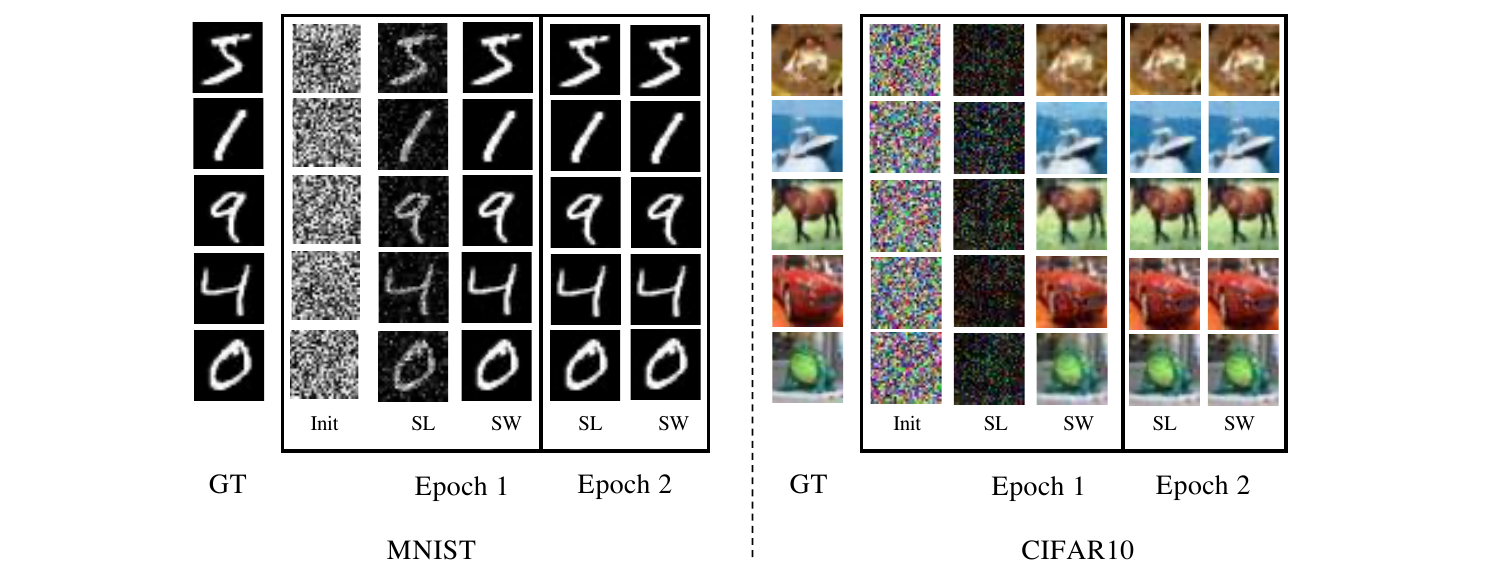}
\caption{\label{fig:progfc}Best viewed zoomed in and in color. Progressive learning of StarNet for feedforward architecture on MNIST and CIFAR10 datasets. Each epoch has two steps, solve w.r.t the latents (denoted as SL) and solve w.r.t the weights (SW). ``Init'' donotes the initial reconstruction given untrained latents and untrained weights. The images start showing high fidelity to the input after the initial epoch.}}
\end{figure*}
\begin{figure*}[t]
\centering{
\includegraphics[width=\linewidth]{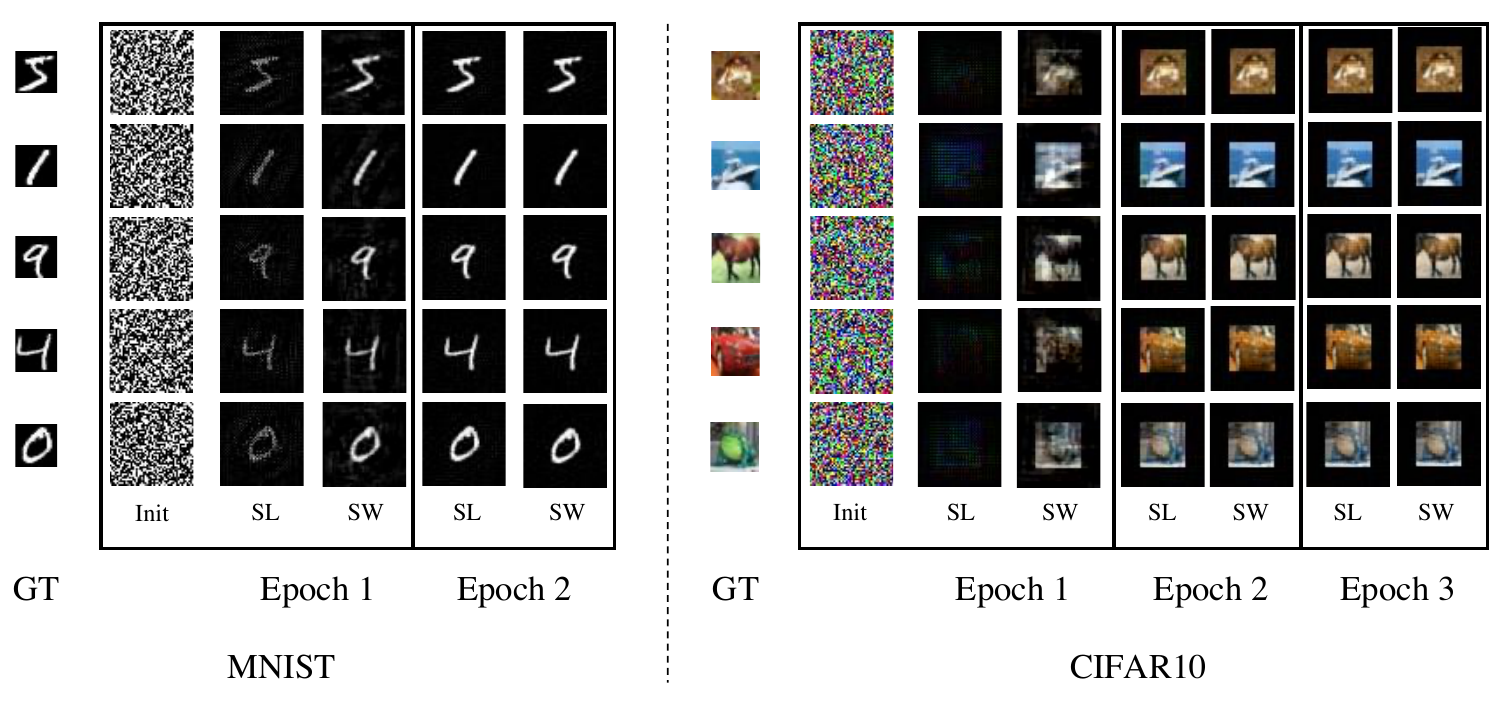}
\caption{\label{fig:progconv}Best viewed zoomed in and in color. Progressive learning of StarNet for convpool architecture on MNIST and CIFAR10 datasets. Each epoch has two steps, solve w.r.t the latents (denoted as SL) and solve w.r.t the weights (SW). ``Init'' donotes the initial reconstruction given untrained latents and untrained weights. The images start showing high fidelity to the input after the initial epoch.}}
\end{figure*}

\section{Discussion}\label{sec:discuss}
We discuss some of the important findings of our experiments in this section. 

\textbf{Performance of a Single Layer Network:} A unique distinction between methods that use free-form parameterization for the latents (i.e. StarNet, VAD), from the methods that use deep layers for inference, is the fact that a single layer network is the most successful in generation. This is due to the fact that the latents are not controlled by an underlying neural structure during the optimization. We posit that a neural architecture imposes a distribution over its range, whereas free-form optimization allows latents to travel everywhere within the latent manifold. Therefore, free-form latents can learn the same output of a deep network (or better ones), if it indeed satisfies the training objective.  

\textbf{Convergence Rate:} On the experimental datasets, StarNet reaches performance plateau mostly within $5$ epochs. Figure \ref{fig:losses} shows the elastic (i.e. $|\cdot| + ||\cdot||$) loss over different epochs for both MNIST and CIFAR10. Note, the loss only reports the error measure and is nowhere used for training the StarNet. For reference, we include the convergence rate of an auto-encoder trained with Adam optimizer with a default learning of $10e-3$. The Auto-Encoder is trained over the elastic objective. 

\textbf{Learning Progress over Different Epochs:} We also study the progress of the generations over different epochs of training. We do this for a single layer feedforward and conv-unpool networks (for MNIST and CIFAR10). The layers are the second layers ($256$ neurons) from Section \ref{sec:training}. The results are shown in Figures \ref{fig:progfc} for feedforward and \ref{fig:progconv} for convolutional. Both architectures show that generations beyond $2$ epochs are already very close to the ground truth. 

\begin{figure*}[t]
\centering{
\includegraphics[width=\linewidth]{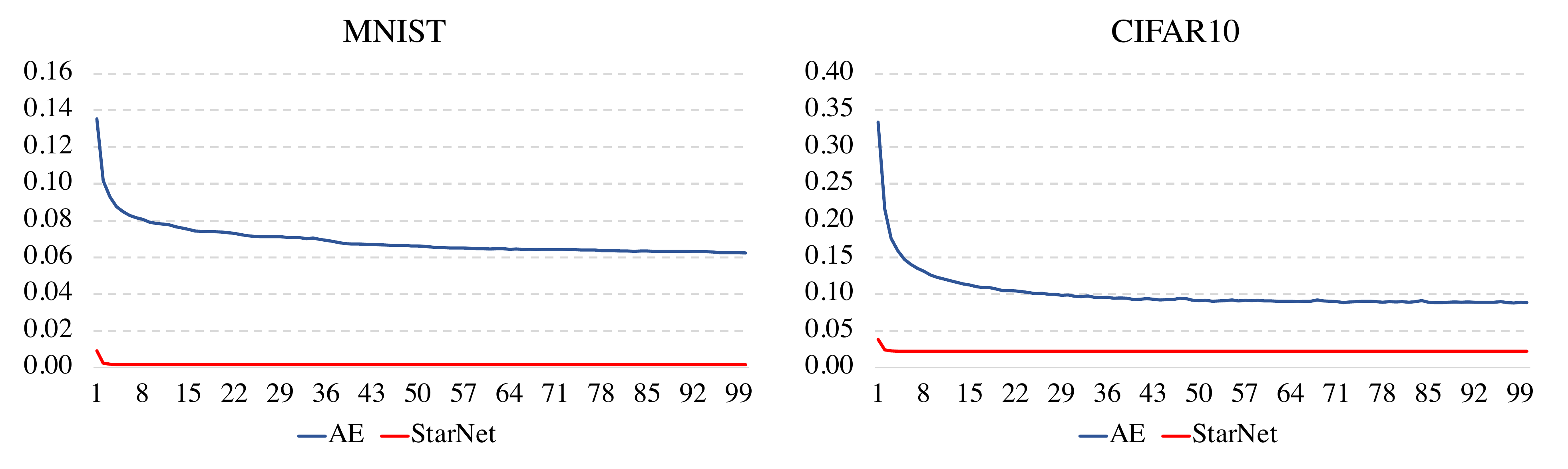}
\caption{\label{fig:losses} Convergence comparison between feedforward StarNet and an Auto-Encoder (AE) over MNIST and CIFAR10. Both decoders are feedfoward architectures with two layers of $128,256$ (AE uses inverse for encoder). AE uses default Adam parameters with learning rate of $10e-3$. Convergence of StarNet is reached within $5$ epochs. }}
\end{figure*}

\textbf{Importance of Both Steps in Learning:} The learning progress in Figures \ref{fig:progfc} and \ref{fig:progconv} demonstrate that while convergence is fast, initial solution (i.e. SL in the first epoch) is not completely successful due to random weights. We also initiate the training w.r.t the weights, and show the results for the MNIST dataset in Figure \ref{fig:wfirst}. Therefore, one cannot train a StarNet by only solving w.r.t either the latents or the weights. Both have to be present. 

\begin{figure*}[t]
\centering{
\includegraphics[width=0.5\linewidth]{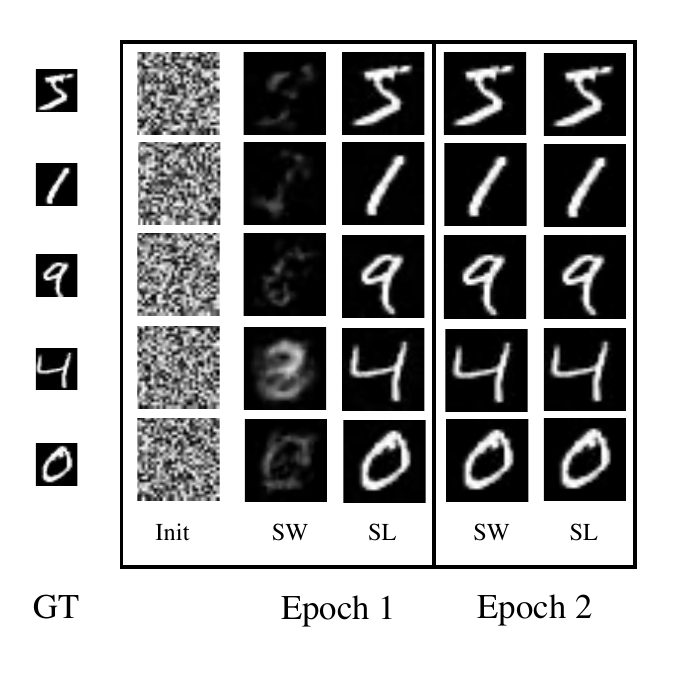}
\caption{\label{fig:wfirst}Best viewed zoomed in. MNIST training for feedforward StarNet. In this case, training starts with solving w.r.t the weights first.}}
\end{figure*}

\textbf{Sampling for Convolutional StarNet Equations:} The system of linear equations in a convolutional StarNet can have too many equations when solving for the parameters of the network. Naturally, this makes calculation of the pseudoinverse (or any other approach for solving system of equations) computationally exhaustive. To mitigate this, we study the effect of random datatpoint sampling on the learning process of the model parameters. For MNIST, Figure \ref{fig:subw} shows that around $500$ datapoint are enough statistics to solve for the weights. Going below that, to $100$ datapoints per epochs seems to make the convergence too slow and the coordinate descent process of learning StarNet unstable. Still, around $500$ equations is a significant reduction from the total number of equations, i.e. $2,880,000$ for full convolutions over MNIST. Note, all models still use all the datapoints for solving w.r.t the latents since datapoints and latents are assumed i.i.d. 

\begin{figure*}[t]
\centering{
\includegraphics[width=0.8\linewidth]{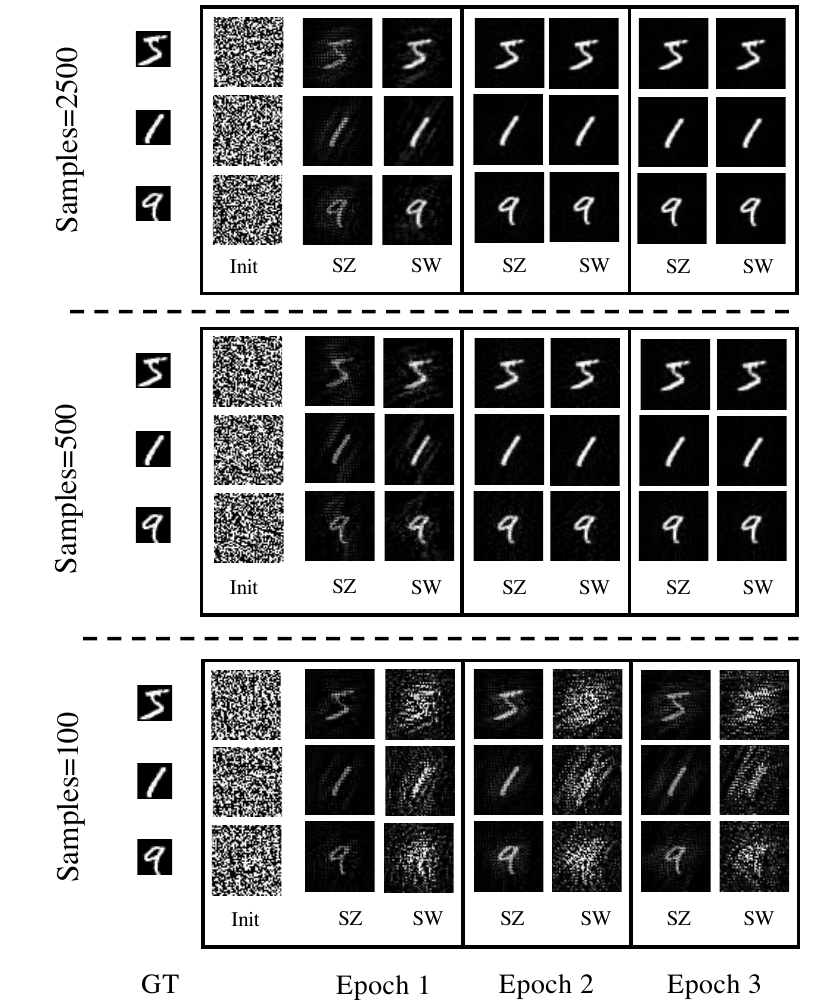}
\caption{\label{fig:subw}Best viewed zoomed in. On the left side, we show the number of samples used for solving the system of linear equations in convolutional StarNet w.r.t the convolution weights. We observe that $500$ images randomly chosen in each epoch are sufficient for training the weights. This sampling serves to mitigate the heavy computations of an otherwise extremely computationally heavy pseudoinverse in convolutional StarNet.}}
\end{figure*}

\textbf{Final Remarks on Scalability:} There are various ways to maximizing computational node utilization when scaling up the learning process of StarNet. First and foremost, when solving both feedforward and convolutional StarNet for the latents, the relation between the number of datapoints and compuational nodes remains linear; i.e. $100$ i.i.d datapoints can go to $100$ computational nodes. Furthermore, for both the feedforward and convolutional cases of StarNet, the solution for the weights can be scaled in two ways: 1) independent weights (rows of weight matrix in feedforward, and kernels in convolution) can be mapped to different computational nodes, 2) in case there are still many more equations than the number of unknowns when solving for the weights (i.e. due to large size of the dataset), a sampling approach can reduce the number of equations to a representative subset, as studied in Figure \ref{fig:subw}. 
\section{Conclusion}

In this paper we presented a new training framework for certain generative architectures. StarNet uses no gradients during training and inference. It relies solely on layer-wise solving of systems of linear equations. This method of training converges fast, and allows for scalability across computational nodes. We present generative experiments over 4 publicly available and well-studied datasets. The generation results show the capabilities of StarNet. 

\bibliographystyle{plain}
\bibliography{citations.bib}

\end{document}